\pdfoutput=1
\documentclass[twoside]{article}

\usepackage[accepted]{aistats2024}
\usepackage{booktabs}
\usepackage{graphicx}
\usepackage{aux_files/macros}
\usepackage{amsmath}

\newcommand{\mysection}[1]{\section{\MakeUppercase{#1}}}
\begin{document}

%

%





\runningtitle{On the Nystr\"om Approximation for Preconditioning in Kernel Machines}

\twocolumn[

\aistatstitle{On the Nystr\"om Approximation for 
Preconditioning\\ in Kernel Machines}
\aistatsauthor{ Amirhesam Abedsoltan \And Parthe Pandit \And Luis Rademacher \And Mikhail Belkin}
\aistatsaddress{CSE, UCSD  \And C-MInDS, IIT Bombay \And Mathematics, UC Davis \And  HDSI, UCSD } ]


\begin{abstract}

Kernel methods are a popular class of nonlinear predictive models in machine learning. Scalable algorithms for learning kernel models need to be iterative in nature, but convergence can be slow due to poor conditioning. 
Spectral preconditioning is an important tool to speed-up the convergence of such iterative algorithms for training kernel models. However computing and storing a spectral preconditioner can be expensive which can lead to large computational and storage overheads, precluding the application of kernel methods to problems with large datasets. 

A Nystr\"om approximation of the spectral preconditioner is often cheaper to compute and store, and has demonstrated success in practical applications. In this paper we analyze the trade-offs of using such an approximated preconditioner.
Specifically, we show that a sample of logarithmic size (as a function of the size of the dataset) enables the Nyström-based approximated preconditioner to accelerate gradient descent nearly as well as the exact preconditioner, while also reducing the computational and storage overheads.
\end{abstract}


\mysection{Introduction}
\begin{table*}
\caption{Trade-offs of using a Nystr\"om approximation to obtain an approximated preconditioner of level $q$ for the preconditioned gradient descent algorithm. Here $\eps$ is a tunable error parameter of the Nystr\"om approximation, and $\curly{\lambda_i^*}$ is the non-increasing sequence of eigenvalues of the integral operator (defined in \eqref{eq:def:integral_operator}) which depends on the data distribution and kernel (but not $n$). 
The comparison with gradient descent is for illustration and based on a heuristic calculation provided in \cref{section:speedup_pgd,eq:speedup_npgd}. The speed-up calculation does not include setup time which is significant for {\sf PGD}, making it impractical for cold-starts.
}\label{tab:speed-ups_lam}
\begin{center}
\begin{tabular}{lccc}
\toprule
\textbf{Iterative algorithm} 
& \textbf{Speed-up over \textsf{GD}} 
& \textbf{Storage}
& \textbf{Setup time}
\\
\midrule
{\sf PGD} (Preconditioned gradient descent)    
& $\displaystyle\frac{\lambda_1^*}{\lambda_q^*
}$ 
& $qn$ 
& $qn^2$
\\[3ex]
{\sf nPGD} ({\sf PGD w/} approximated preconditioner)
& ${\displaystyle\frac{\lambda^*_1}{\lambda^*_{q}}}\cdot\displaystyle\frac1{(1+\eps)^4}$
& $q\cdot O(\frac{\log^4n}{\eps^4})$
& $q\cdot O(\frac{\log^8n}{\eps^8})$
\\[2ex]
\bottomrule
\end{tabular}
\end{center}
\end{table*}

\begin{table}[t]
    \caption{Notation}\label{tab:notation_intro}
    \begin{center}
    \begin{tabular}{ll}\toprule
        $n$ & number of training samples\\
        $q$ & level of preconditioner\\
        $s$ & number of Nystr\"om samples\\
        $\kappa_1$ & 
        condition number {\sf w/o} preconditioning\\
        $\kappa_q$ &
        condition number {\sf w/} preconditioning\\
        $\kappa_{s,q}$ & condition number {\sf w/} approx.\ preconditioner\\
        \bottomrule
    \end{tabular}
    \end{center}
\end{table}
Deep neural networks have consistently delivered remarkable performance across a wide range of machine learning tasks, setting unprecedented benchmarks, and reshaping the landscape of data modelling. Recent findings have drawn a connection between certain architectures of these networks (such as wide neural networks) and the more classical kernel methods. In the infinite width limit, neural networks converge to a specific form of kernels known as Neural Tangent Kernels (NTKs) \cite{jacot2018neural}. This insight has reignited enthusiasm for studying and applying kernel methods for large-scale machine learning problems. Moreover, emerging research suggests that in certain contexts, such as problems with limited data, kernel methods can surpass neural networks in performance \citep{arora2019harnessing, shankar2020neural,bietti2020deep}. Additionally, kernel methods have been instrumental in elucidating the complex feature learning dynamics within deep neural networks \citep{radhakrishnan2022feature,beaglehole2023mechanism}.  This intersection of classical kernel methods and deep neural networks offers a promising direction for further research towards a deeper understanding of contemporary model architectures. 

This motivates the need for scalable approaches for training kernel models on large datasets.
Off-the-shelf linear system solvers are often not scalable. 
Consequently, new algorithms and implementations for scalable training of kernel models have been proposed over the last few years. 
These include
\cite[{\small \sc Pegasos}]{shalev2007pegasos}, 
\cite[{\small \sc Nytro}]{camoriano2016nytro},  
\cite[{\small \sc Falkon}]{meanti2020kernel, rudi2017falkon}, 
\cite[{\small \sc KeOps}]{charlier2021kernel}, 
\cite[{\small \sc GPyTorch}]{gardner2018gpytorch}, 
\cite[{\small \sc GPFlow}]{matthews2017gpflow}, 
\cite[{\small \sc EigenPro}]{ma2017diving},  
\cite[{\small \sc EigenPro2}]{ma2019kernel} and \cite[{\small \sc EigenPro3}]{abedsoltan2023toward}. 
Some of these can also take advantage of modern hardware such as graphical processing units (GPUs) effectively.

Consider a dataset $(x_1,y_1), \dotsc, (x_n, y_n) \in\RR^d \times \RR$ and a positive definite kernel $K:\Real^d\times\Real^d\rightarrow\Real$.
Let $\Hilbert_K$ be the reproducing kernel Hilbert space (RKHS) corresponding to $K$.
We are interested in the following model: 
the minimum $\Hilbert_K$-norm solution to the quadratic loss optimization problem
\begin{align}\label{eq:square_loss}
    \lmin_{f\in\Hilbert_K}\textsf{L}(f)=\frac1{2n}\sum_{i=1}^n (f(x_i)-y_i)^2,
\end{align}
The number of iterations required for an iterative optimization algorithm, such as gradient descent ({\sf GD}), to converge depends on the condition number of a specific operator known as the empirical covariance operator $\mcK:\Hilbert_K\rightarrow\Hilbert_K$.
Operator $\mcK$ is the Hessian $\grad^2_{\! f}{\msf L}$ of the optimization problem above (see \Cref{sec:main_result}).
For our purposes, we define the condition number of an operator $\mcA:\Hilbert \to \Hilbert$ by
\begin{align}\label{def:kappa}
 \kappa(\mc A) := \frac{\sup_{\norm{f}=1} \norm{\mc A f}}{\inf_{f\in\nullspace(\mc A)^\perp ,\norm{f}=1} \norm{\mc A f}},
\end{align}
where $\nullspace(\cdot)$ is the null-space. A higher condition number means that the iterative algorithm requires more iterations to converge.

In this work we will generally be interested in the condition number of finite-rank operators. For a rank-$k$ operator $\mcA$ our definition can be written as $\kappa(\mcA) = \sigma_1(\mcA)/\sigma_k(\mcA)$, where $\sigma_1(\mcA) \geq \sigma_2(\mcA) \geq \dotsb \geq \sigma_k(\mcA) > 0$ are the singular values of $\mcA$ (defined by operator theory, or matrix theory expressing the operator as a matrix in terms of orthonormal bases).

In practice, condition number $\kappa_1:= \kappa(\mcK) $ is often large. A common approach to overcome this challenge is to use a \emph{preconditioner} operator, $ \mcP:\Hilbert_K\rightarrow\Hilbert_K $. 
This enables solving the same problem with a modified condition number  $\kappa(\mcP\mycirc\mcK)$ instead of $\kappa(\mcK)$, for example in preconditioned gradient descent ({\sf PGD}). 
Many types of preconditioners have been studied in the literature on numerical methods for learning kernel models (see \cite{cutajar2016preconditioning} for a comparison). 
Approximated preconditioners have also been studied, see for example \cite{avron2017faster} and \cite{rudi2017falkon}.

\paragraph{Spectral preconditioning.} 
In this paper we focus on a specific type of preconditioner. 
For the operator $\mcK$, we say $\mcP_q$ is a \textit{spectral preconditioner} if the top-$q$ eigenvalues of $\mcP_q\mycirc\mcK$ are equal to the $q^{\textrm{th}}$ largest eigenvalue of $\mcK$ (and the rest of the eigenvalues stay unchanged), and we refer to $q$ as the \textit{level} of the preconditioner.
The resulting condition number is $\kappa_q:=\kappa(\mc P_q\mycirc\mc K)$.
This preconditioner works well in practice as shown in \cite{ma2017diving}, which introduced EigenPro --- a state-of-the-art method for learning kernel models based on stochastic gradient descent ({\sf SGD}).
\paragraph{Nystr\"om approximation.} 
However, computing and storing a spectral
preconditioner $\mc P_q$ is challenging.
It requires an additional $O(qn^2)$ time \cite{NIPS2000_19de10ad} 
and $O(qn)$ memory \cite{ma2017diving}. 
See the discussion following \Cref{prop:eigenfunction_representation} in  \Cref{section:speedup_pgd} of our paper for details.

The Nyström extension is a technique used to approximate \(\mathcal{K}\) with a smaller number of samples, denoted by \(s\), where \(s \ll n\). Subsequently, we can approximate spectral preconditioner \(\mathcal{P}_q\) using these $s$ samples and denote it as \(\mathcal{P}_{s,q}\). This method requires time \(O(qs^2)\) and storage \(O(qs)\).  

Just like {\sf PGD} we can use this approximated preconditioner to accelerate {\sf GD}. We refer to this version of {\sf PGD} with a Nystr\"om approximated preconditioner as ({\sf nPGD}). The convergence of {\sf nPGD} depends on
\begin{align}
\kappa_{s,q}:= \kappa(\mcP_{s,q}^{\frac12} \mycirc \mcK \mycirc \mcP_{s,q}^{\frac12}),    
\end{align}
as shown later in \Cref{sec:main_result}. 
In fact, EigenPro2.0 \cite{ma2019kernel} applies $\mcP_{s,q}$ to improve the scalability of EigenPro, which used $\mcP_q$, the exact preconditioner.

Prior to our work, there was no rigorous way to choose $s$, the size of the Nystr\"om sample, so that one can guarantee a speed-up when using the approximated preconditioner. 

\subsection{Main contribution}
We show that \textsf{nPGD} can achieve nearly the same speed-up as \textsf{PGD} does 
over \textsf{GD}, while only requiring a polylogarithmic storage overhead and setup-cost. See Table \ref{tab:speed-ups_lam} for the trade-off of using a Nystr\"om approximated preconditioner.

In order to show this, we analyze how many Nystr\"om samples are sufficient to achieve a particular approximation quality of approximated preconditioner.
Our main result (\Cref{theorem:main}) shows that for a given $\eps > 0$, we can achieve
$
\kappa_{s,q}
\leq(1+\eps)^4\kappa_q
$
 if the number of Nyström samples satisfies
$
s = \Omega\bigl(\frac{\log^4 n}{\eps^4}\bigr)
$.

While we only provide explicit speed-up computations for versions of preconditioned \textsf{GD}, our main result can be applied to guarantee speed-ups for other spectrally preconditioned algorithms such as conjugate gradient and Nesterov's accelerated gradient methods.

\paragraph{Organization:} 
In \cref{sec:prelim} we discuss preliminaries needed to state our main result. 
\Cref{sec:main_result} details the problem formulation, and \Cref{sec:maintheorem} provide the statement of the main result in \cref{theorem:main}, followed by its proof in \cref{sec:proof_main}. The intermediate lemmas needed in the proof of \cref{theorem:main} are detailed in \cref{sec:operator}.
\mysection{Preliminaries}
\label{sec:prelim}

\textbf{Notation:} 
We denote by $\Hilbert$ a separable Hilbert space, and by $\Hilbert_K$ a reproducing kernel Hilbert space (RKHS) associated with a symmetric positive definite kernel $K:\Real^d\times\Real^d\rightarrow\Real$. We assume that $K$ is continuous and bounded.
For any point $x\in \Real^d$, the function $K(x,\cdot):\Real^d\rightarrow\Real$ belongs to $\Hilbert_K$. 
Let $L(\Hilbert)$ denote the set of bounded linear operators from $\Hilbert$ to $\Hilbert$.
For operators and when defined, we will denote the Hilbert-Schmidt norm by $\hsnorm{\cdot}$ and the operator norm by $\opnorm{\cdot}$. Without any subscript, $\snorm{\cdot}$ denotes the norm of $\mcH$ or $\mcH_K$ when there is no confusion.

\paragraph{Square-root operator:} For a self-adjoint positive semidefinite operator $\mc A$, we denote by $\mcA^{\frac12}$ the unique self-adjoint square-root of $\mc A$, such that $\mcA^{\frac12}\mycirc\mcA^{\frac12}=\mc A$. Thus if $\mc A$ has an eigen-decomposition $\mc A=\sum_i\sigma_i\phi_i\otimes \phi_i$ (with $\sigma_i\geq 0$), we define
\begin{align*}
    \mcA^{\frac12}:=\sum_i\sqrt{\sigma_i}~\phi_i\otimes \phi_i.
\end{align*}
\paragraph{Eigenvalue thresholding:}
We will need the following elementary version of functional calculus.
It will be convenient to represent preconditioner operators as the result of modifications to the eigenvalues of a covariance operator.
We will need an operation that replaces every eigenvalue below a threshold by the threshold.

For self-adjoint finite rank $\mcA \in L(\Hilbert)$ with eigendecomposition $\mcA = \sum_{i \in I} \mu_i \phi_i \otimes \phi_i$, $\mu_i \neq 0$ and for a continuous function $h:\RR \to \RR$, the corresponding operator function is 
$
 h(\mcA) = \sum_{i \in I} h(\lambda_i) \phi_i \otimes \phi_i + h(0) \proj_{\nullspace(\mcA)}$.  
We will need the following thresholding function:
$h_{\alpha}(x) = \lmax\{x, \alpha \}$. When $\Hilbert$ is finite dimensional (and more generally, but this is all we need), we have
$
  h_\alpha(\mcA) = h_0(\mcA-\alpha \mcI) + \alpha \mcI$.  

\mysection{Problem Setup}\label{sec:main_result}
Let $\rho$ be a probability measure on $\Real^d$.
Let $X_n=\curly{x_i}_{i=1}^n$ be the training samples drawn i.i.d.~from $\rho$. 
Let $(i_1,i_2,\ldots, i_s)$ be a random tuple of $s\leq n$ distinct indices $1\leq i_k\leq n$ chosen independently of $X_n$, and define $X_s=\curly{x_{i_k}}_{k=1}^s$ as the Nystr\"om samples.
Namely, for our results to hold the distribution of the tuple can be arbitrary as long as it is independent of $X_n$ and the indices are all distinct with probability $1$.
Two natural choices of this distribution are \emph{uniform sampling in $[n]$ without replacement} and \emph{any fixed tuple}, e.g.\ $i_k = k$.

Consider $n$ targets $y_1,y_2,\ldots,y_n\in\Real$. Assume that if $x_i=x_j$, then $y_i=y_j$. Our model is the unique minimum $\mcH_K$-norm solution $f^*$ to problem \eqref{eq:square_loss}. 
Assume $K$ is a bounded, continuous, symmetric positive definite kernel.
Note that under our assumptions on $K$ we have that $f^*$ interpolates the data, namely the optimal value of problem \eqref{eq:square_loss} is $0$.

Define the following operators in $L(\Hilbert_K)$:
\begin{align}\label{eq:def:integral_operator}
    &\mcT:= \int_{\Real^d} K(x,\cdot)\otimes K(x,\cdot) \rho(\dif x)  
    \\\label{eq:def:covariance_operator}
    &\mcK := \displaystyle\frac1n\sum_{i=1}^n K(x_i,\cdot)\otimes K(x_i,\cdot)
    \\\label{eq:def:subsampled_covariance_operator}
    &\mcK':=\displaystyle\frac1s\sum_{i=1}^s
    K(x_{i}',\cdot)\otimes K(x_{i}',\cdot).
\end{align}
For $f\in\Hilbert_K$, due to the reproducing property of kernel $K$, 
the above operators act as follows: $\mc Tf(x) = \int K(x,z)f(z)\rho(\mathrm{d} z)$, $\mc Kf(x) = \frac1n\sum_{i=1}^n K(x,x_i)f(x_i)$, and $\mc K'f(x) = \frac1s\sum_{i=1}^sK(x,x_{i}')f(x'_{i})$.

By the linearity of the trace we can show that
\begin{align}\label{eq:trace_bound}
    \trace(\mc T)\leq \beta(K),\quad\text{and}\quad
    \trace(\mc K)\leq \beta(K),
\end{align}
where we define
\begin{align}\label{eq:def:beta}
    \beta(K) := \max{x\in \Real^d} K(x,x).
\end{align}
In \eqref{eq:trace_bound} we have used the fact that $\trace(K(x,\cdot)\otimes K(x,\cdot))=\inner{K(x,\cdot),K(x,\cdot)}_{\Hilbert_K}=K(x,x).$

%
%

Next, let $\psi_i^*$ be an eigenfunction of $\mc T$ with eigenvalue $\lambda_i$, i.e., $\mc T\psi_i^*=\lambda_i\psi_i^*.$  Similarly denote by $(\lambda_i, \psi_i)$ the eigenpairs of $\mc K$ and by $(\lambda_i', \psi_i')$ the eigenpairs for $\mc K'$. 
    Note that $\mc T$ is a compact operator whereas $\mc K$ and $\mc K'$ are empirical approximations of $\mc T$. Furthermore, $\mc T,\mc K,$ and $\mc K'$ have non-negative eigenvalues which we assume are ordered as $\lambda_1^*\geq \lambda_2^*\geq \dotsb \geq 0$, and similarly $\lambda_1\geq \lambda_2 \geq \dotsb\geq\lambda_n\geq 0,$ and $\lambda'_1\geq \lambda'_2 \geq \dotsb \geq\lambda'_s\geq 0.$ Hence we have eigen-decompositions for these operators as written below.
    \begin{align*}
        \mc T = \sum_{i=1}^{\infty} \lambda_i^*\cdot &\psi_i^*\otimes\psi_i^*, \qquad
        \mc K = \sum_{i=1}^{n} \lambda_i\cdot \psi_i\otimes\psi_i
        , \\\qquad &\mc K' = \sum_{i=1}^{s} \lambda_i'\cdot \psi_i'\otimes\psi_i'.
    \end{align*}
    For a fixed $q\leq s,$ we define the following preconditioners with the assumption that $\lambda_{q}, \lambda_{q}'>0$ as below,
    \begin{align}    \label{eq:preconditioner}
    &\mcP_q := \mcI - \sum_{i=1}^{q-1} \round{1-\frac{\lambda_{q}}{\lambda_i}}\psi_i\otimes\psi_i 
    \\\label{eq:nystrom_preconditioner}
    &\mcP_{s,q} := \mcI - \sum_{i=1}^{q-1} \round{1-\frac{\lambda_{q}'}{\lambda_i'}}\psi_i'\otimes\psi_i'.
\end{align}
With the above definition, one can verify that
\begin{align}
    \sqrtop{\mcP}_{s,q} = \mc I - \sum_{i=1}^{q-1}\round{1-\sqrt{\frac{\lambda'_{q}}{\lambda'_i}}}\psi_i'\otimes\psi_i'.
\end{align}

\subsection{Gradient Descent and preconditioning} 
\label{section:gradient_descent}

Define $\mrm b:=\frac1n\sum_{i=1}^n K(\cdot,x_i)y_i\in\Hilbert_K,$ and note that $\mrm b=\mc Kf^*$.

Due to the reproducing property of the RKHS, we have $f(x)=\inner{f,K(x,\cdot)}_{\Hilbert_K}$ for all $f\in\Hilbert_K$ and for all $x\in\Real^d,$ whereby the Fr\'echet derivative $\grad_{\! f}f(x)$ equals $K(\cdot,x)$ for all $f\in\Hilbert_K$ and $x\in\Real^d.$ Using a similar argument, observe that $\grad_{\!f}\textsf{L}(f)=\mc Kf-\mrm b=\mc K(f-f^*)$, and $\mc K=\grad_{\!f}^2\textsf{L}(f)$ is the Hessian of the optimization problem \eqref{eq:square_loss}.

One step of Gradient Descent (GD) with learning rate $\eta$ is given by:
\begin{align}
    f^t &= f^{t-1} - \eta \nabla_{f} \mathsf{L}(f^{t-1})\tag{\sf GD}\\
    &=(\mathcal{I} - \eta\mathcal{K})f^{t-1}+\eta\mrm b.\nonumber
\end{align}
where we have used the fact that
$\nabla_{f} \mathsf{L}(f) = \sum_{i=1}^n(f(x_i)-y_i)K(x_i,\cdot) = 
\mathcal{K}f-\mrm b.$

Then, for the purpose of convergence analysis, one can add and subtract $f^*$ and verify that the above equation can be rewritten as,
\begin{align*}
    f^t - f^* 
    &=(\mathcal{I} - \eta\mathcal{K})(f^{t-1}-f^*)
\end{align*}
where we have used the fact that $\mrm b=\mc K f^*.$

The convergence rate of this iteration is governed by the condition number of the operator $\mc K.$ 
Standard proof techniques for gradient descent, e.g. \cite[Theorem 3.6]{garrigos2023handbook} allow us to show that with the choice $\eta=1/\kappa(\mc K)=1/\kappa_1$ we can find $f^t$ such that
\begin{align*}
    \norm{f^t - f^*}^2_{\Hilbert_K} \leq \tau\norm{f^0 - f^*}^2_{\Hilbert_K}
\end{align*}
in $t=\kappa_1\log\round{\frac1\tau}$ iterations.

\paragraph{Preconditioned gradient descent:} Using the preconditioner $\mcP_q$ with level $q$, the update equations for preconditioned gradient descent with step size $\eta_q$ are
\begin{align}
    f^t &= f^{t-1} - \eta_q \mcP_q\nabla_{f} \mathsf{L}(f^{t-1}) \tag{\sf PGD}\\
    &=(\mathcal{I} - \eta_q\mcP_q\mycirc\mathcal{K})f^{t-1}+\eta\mcP_q\mrm b.\nonumber
\end{align}
whose convergence rate is determined by the condition number of $\mc P_q\mycirc\mc K$ (since $\mc P_q$ and $\mc K$ commute and therefore $\mc P_q \mycirc \mc K$ is self-adjoint). A similar argument as for {\sf GD} shows that {\sf PGD} is equivalent to
\begin{align*}
    f^t - f^* = (\mc I - \eta_q\mc P_q\mycirc\mc K)(f^{t-1} -f^*).
\end{align*}
With the choice $\eta_q=1/{\kappa(\mc P_q\mycirc\mc K)}=1/\kappa_q$, the iteration converges within $\kappa_q\log\round{\tfrac1\tau}$ steps.

\paragraph{Approximated preconditioner:}
Preconditioning with an approximated preconditioner $\mcP_{s,q}$ modifies the update rule to:
\begin{align}
f^t &= f^{t-1} - \eta_{s,q} \mcP_{s,q}\nabla_{f} \mathsf{L}(f^{t-1}) \tag{\sf nPGD}\\
&= (\mathcal{I} -\eta_{s,q} \mcP_{s,q} \mycirc \mathcal{K})f^{t-1}+\eta_{s,q}\mc P_{s,q}\mrm b.\nonumber
\end{align}
which can similarly be shown to be equivalent to
\begin{align*}
    f^t-f^*=(\mathcal{I} -\eta_{s,q} \mcP_{s,q} \mycirc \mathcal{K})(f^{t-1}-f^*).
\end{align*}
Observe that $\sqrtop{\mc P_{s,q}}$ is invertible, whereby introducing a change of variables, $g^{t} = \sqrtopinv{\mcP_{s,q}} f^t$, and $g^*=\sqrtopinv{\mcP_{s,q}} f^*$, we can show {\sf nPGD} is equivalent to
\begin{align*}
g^{t}-g^* = (\mathcal{I} - \eta_{s,q}\sqrtop{\mcP_{s,q}} \mycirc \mathcal{K} \mycirc 
\sqrtop{\mcP_{s,q}})(g^{t-1}-g^*)
\end{align*}
whose convergence rate is determined by the condition number of the self-adjoint operator $\sqrtop{\mcP_{s,q}} \mycirc \mathcal{K} \mycirc 
\sqrtop{\mcP_{s,q}}.$ Specifically with $\eta_{s,q}=1/\kappa(\sqrtop{\mc P_{s,q}}\mycirc\mc K\mycirc\sqrtop{\mc P_{s,q}})=1/\kappa_{s,q}$, the iteration converges in $\kappa_{s,q}\log\round{\frac1\tau}$ steps.
\begin{remark}
We consider $\sqrtop{\mcP_{s,q}} \mycirc \mathcal{K} \mycirc 
\sqrtop{\mcP_{s,q}}$  to deal with the fact that $\mc P_{s,q}\mycirc\mc K$ is not self-adjoint. This issue does not occur for $\mc P_{q}\mycirc\mc K$, since $\mc P_{q}$ and $\mc K$ commute whereby $\mc P_{q}\mycirc\mc K$ is self-adjoint.
\end{remark}
{Our main result in the next section  provides sufficient conditions on $s$, under which we can obtain a multiplicative approximation of $\kappa(\sqrtop{\mcP_{s,q}} \mycirc \mathcal{K} \mycirc 
\sqrtop{\mcP_{s,q}})$ in terms of $\kappa(\mcP_{q} \mycirc \mathcal{K})$.}

\subsection{Speed-up of {\sf PGD} over {\sf GD}}\label{section:speedup_pgd}

Here we provide a summary of the computational and storage costs of three algorithms --- gradient descent ({\sf GD}), preconditioned gradient descent ({\sf PGD}), and preconditioned gradient descent with a Nystr\"om approximated preconditioner ({\sf nPGD}).

The time for running {\sf GD} is
\begin{align}\label{eq:time_gd}
    T_{\sf GD} = {\kappa_1 n^2\log\round{\tfrac1\tau}}
\end{align}
where $n^2$ is the time per iteration for calculating $\grad_{\!f}\textsf{L}(f^t)$ and $\kappa_1\log\round{\tfrac1\tau}$ is the number of iterations required to achieve error $\tau.$ Note the $n^2$ complexity arises since we need to compute $\grad_{\!f}\msf L(f^t)=\mc K f^t.$ 
If we use the basis expansion $f^t=\sum_{i=1}^n K(x_i,\cdot)\alpha_i$, we have $\mc Kf^t=\sum_{i=1}^n\sum_{j=1}^n K(x_j,\cdot)K(x_i,x_j)\alpha_i$, which requires $n^2$ steps. One can argue that the quadratic complexity is optimal under our assumptions.

Similarly, the time taken for preconditioned gradient descent is
\begin{align}\label{eq:time_pgd}
    T_{\sf PGD} = {\kappa_q\round{n^2 + 2qn}\log\round{\tfrac1\tau}},
\end{align}
where $2nq$ is the per iteration overhead to apply the preconditioner, i.e. calculating $\grad_{\!f}\textsf{L}(f^t)\mapsto\mc P_{q}\grad_{\!f}\textsf{L}(f^t)$. Thus the speed-up of {\sf PGD} over {\sf GD} is
\begin{align}\label{eq:speedup_pgd}
    \frac{T_{\sf GD}}{T_{\sf PGD}} = \frac{\kappa_1}{\kappa_q + q/n} \approx \frac{\kappa_1}{\kappa_q} = \frac{\lambda_{1}}{\lambda_q} \approx \frac{\lambda_1^*}{3\lambda_q^*}
\end{align}
when $q=O(1)$, and the last approximation holds because of a concentration argument that allows us to control $\abs{\lambda_i-\lambda_i^*}$ for large enough $n,$ see equation \eqref{eq:populationlambda}.

Representing the spectral preconditioner $\mc P_q\equiv \curly{\psi_i}_{i=1}^q$ requires storing $q$ vectors of length $n$ as explained in the following lemma. 
\begin{proposition}\label{prop:eigenfunction_representation}
    Let $\e=(e_{j})\in\Real^n$ be an eigenvector of the matrix $\round{K(x_j,x_k)}_{1\leq i,j\leq n},$ with eigenvalue $n\lambda$, then
    $\psi = \sum_{j=1}^n e_j K(\cdot, x_j)$
    is an eigenfunction of $\mc K,$ with eigenvalue $\lambda.$
\end{proposition}
\begin{proof}
    Observe that if $\psi$ is defined as above, $n\mc K\psi=n\sum_{j=1}^ne_j\mc K K(\cdot,x_j)=\sum_{j,k=1}^ne_jK(x_k,x_j)K(\cdot,x_k)=\sum_{k=1}^n K(\cdot, x_k)\round{\sum_{j=1}^n K(x_k,x_j)e_j}$. The term in the parenthesis equals $n\lambda e_{k},$ since $\e$ is an eigenvector. 
\end{proof}
Thus storing the preconditioner $\mc P_q$ requires storing $qn$ floats. Additionally, the cost of calculating the $q$ eigenvectors needed to represent $\curly{\psi_i}_{i=1}^q$ is $O(qn^2)$.

Similar to \eqref{eq:time_pgd}, for {\sf nPGD} we have
\begin{align}\label{eq:time_npgd}
    T_{\sf nPGD} = 
    {\kappa_{s,q}\round{n^2 + 2sn + 2sq}\log\round{\tfrac1\tau}},
\end{align}
where $2ns+2nq$ is the per iteration overhead in calculating $\grad_{\!f}\textsf{L}(f^t)\mapsto\mc P_{s,q}\grad_{\!f}\textsf{L}(f^t)$. Hence we get that
\begin{align}\label{eq:speedup_npgd}
    \frac{T_{\sf nPGD}}{T_{\sf GD}} = \frac{\kappa_1}{\kappa_{s,q} + s/n + sq/n^2} \approx \frac{\kappa_1}{\kappa_{s,q} + s/n }
\end{align}
when $q=O(1).$

An argument similar to the one following \Cref{prop:eigenfunction_representation} for {\sf PGD} shows that
representing $\mc P_{s,q}\equiv\curly{\psi_i'}_{i=1}^q$ requires storing eigenvectors of the $s\times s$ matrix $\round{K(x_j',x_k')}_{1\leq j,k\leq s}$, which requires storing $qs$ floats, while computing $q$ eigenvectors requires time $O(qs^2).$ 

\mysection{Main Result}\label{sec:maintheorem}
The main result of the paper provides a multiplicative bound on the condition number for the {\sf nPGD} iterations in terms of the condition number for {\sf PGD}. To that end, recall the definitions of $\beta(K)$ from \eqref{eq:def:beta}, and the condition number $\kappa$ of an operator from \eqref{def:kappa}.

We are ready to state the main result of this paper.

\begin{theorem}\label{theorem:main}
    Assume that kernel $K:\RR^d \times \RR^d \to \RR$ is symmetric positive definite, continuous and bounded.
    Consider the operators $\mcK$, $\mc P_q$, and $\mc P_{s,q}$ defined in equations \eqref{eq:def:covariance_operator},\eqref{eq:preconditioner}, and \eqref{eq:nystrom_preconditioner}, respectively. 
    Let $q \geq 1$ be such that $\lambda_q^*>0$. For any $\varepsilon>0$ and $\delta\in(0,1)$,  we have
    \begin{align*}
    \kappa(\sqrtop{\mcP_{s,q}}
    \mycirc
    \mcK\mycirc\sqrtop{\mcP_{s,q}})
    \leq 
    (1+\eps)^4
    \operatorname{\kappa}(\mcP_q\mycirc\mcK)
    \end{align*}
    with probability at least $1-\delta$ (over the randomness of $X_n$ and $X_s$) when
    \begin{align}\label{eq:s_bound}
        s \geq \lmax &\biggl\{c_1 C_{K,q}^2,\frac{c_2 C_{K,q}^4\log^4( n + 1)}{\varepsilon^4}\biggr\}\log\round{\tfrac4\delta}, 
    \end{align}
    where $C_{K,q}:=\frac{\beta(K)}{\lambda_q^*}$, and $c_1, c_2$ are universal constants.
\end{theorem}
\subsection{Speed-up of {\sf nPGD} over {\sf GD}}
\label{section:speedup_npgd} 
    Consider the fact that the per iteration complexity of all 3 algorithms {\sf GD}, {\sf PGD} and {\sf nPGD} are $O(n^2)$. Consequently the speed-up in {\sf PGD} and {\sf nPGD} over {\sf GD} arises due to fewer number of iterations.

    The level of the preconditioner $q$ is implicitly constrained due to \eqref{eq:s_bound}. For simplicity, we will provide the analysis for $q=O(1)$, i.e., it cannot depend on $n.$ 
    However, $\lambda_i$s also depend on $n$. To keep the dependence on $n$ straightforward we can use the bounds $\frac{\lambda_i^*}{2}\leq \lambda_i\leq \frac{3\lambda_i^*}{2}$ (see \eqref{eq:populationlambda}), which hold under the assumptions of \cref{theorem:main}. This helps because $\lambda_i^*$s are independent of $n$.

    Theorem \ref{theorem:main} allows $\delta$ to depend on $n$, but we will assume $\delta=O(1)$ for simplicity. We also have $C_{K,q}=O(1)$ since we assume $q=O(1)$. Consequently $s=\Omega(\frac{\log^4n}{\eps^4})$ is optimal. 
    Hence, following \eqref{eq:speedup_npgd}, we can write
    \begin{align*}
        \frac{T_{\sf GD}}{T_{\sf nPGD}}=\frac{\kappa_1}{\kappa_{s,q}+s/n}\approx\frac{\kappa_1}{\kappa_{s,q}}\approx\frac{\kappa_1}{\kappa_q(1+\eps)^4}\approx\frac{\lambda_1^*}{3\lambda_q^*(1+\eps)^4}
    \end{align*}
    following a concentration argument similar to \eqref{eq:speedup_pgd} that controls $\abs{\lambda_i'-\lambda_i^*}$ under the assumptions on $s$, see \eqref{eq:populationlambda}. 
    We omit the constant $3$ in the denominator since it becomes arbitrarily close to $1$ for large $n$.

\mysection{Proof of Main Result}
\label{sec:proof_main}


Recall the definitions of the operators $\mc T, \mc K, \mc K',\mc P_q$ and $\mc P_{s,q}$ in equations (\ref{eq:def:integral_operator}--\ref{eq:nystrom_preconditioner}).


{The proof proceeds in three key steps: 
(i)~We start by stating (in Proposition \ref{propos:operator_disturb}) that $\opsnorm{\sqrtop{\mc K}-\sqrtop{\mc K'}}$ can be made arbitrarily small if $s$ and $n$ are large enough.
This follows from a concentration result in \cite{rosasco2010learning} that we state for our case in \cref{cor:rosasco}.
(ii)~Next, in Proposition \ref{prop:ratio_to_condition_number} we show why $\hfrac{\snorm{\sqrtop{\mc P}_qf}}{\snorm{\sqrtop{\mc P}_{s,q}f}}$ being close to $1$ for all $f$ is sufficient to prove the claim in \Cref{theorem:main}. 
(iii)~Finally we show that in fact $\opsnorm{\sqrtop{\mc K}-\sqrtop{\mc K'}}$ being small is sufficient for $\hfrac{\snorm{\sqrtop{\mc P}_qf}}{\snorm{\sqrtop{\mc P}_{s,q}f}}$ to be close to $1$.
}

We first start by showing that $X_s$, just like $X_n$, is an i.i.d.~sample from $\rho,$ whereby we can apply the same concentration arguments for both.
\begin{lemma}
Let $x_1, \dotsc, x_n$ be i.i.d.~random variables drawn from $\rho$, a probability distirbution over measurable space $(\Real^d, \beta)$.
Let $i_1, \dotsc, i_s \in [n]$ be random according to some joint probability distribution such that $i_1, \dotsc, i_s$ are all distinct almost surely, and such that $i_1, \dotsc, i_s$ is independent of all $x_i$.
Then $x_{i_1}, \dotsc, x_{i_s}$ are i.i.d., each distributed as $\rho$.
\end{lemma}
\begin{proof}
Let $\rho^s$ denote the product measure on $(\Real^{d})^s$. Let $A \in (\Real^{d})^s$ and measurable. 
Conditioning on $i_1, \dotsc, i_s$:
\begin{align*}
&\Prob\bigl((x_{i_1}, \dotsc, x_{i_s}) \in A\bigr) \\
&= \Exp \bigl(\Prob((x_{i_1}, \dotsc, x_{i_s}) \in A \mid i_1, \dotsc, i_s) \bigr)\\ 
&\overset{\rm (a)}= \Exp \bigl( \rho^s(A) \bigr) = \rho^s(A),
\end{align*}
where (a) holds because $i_1, \dotsc, i_s$ are distinct and independent of $x_1, \dotsc, x_n$. 
\end{proof}

We can now apply a concentration inequality to bound error between $\mc K$ and $\mc T$, and between $\mc K'$ and $\mc T.$

\begin{corollary}[{of \cite[Theorem 7]{rosasco2010learning}}]\label{cor:rosasco}\ 
    With probability at least $1-\frac\delta2$, with respect to the randomness of $X_n$, we have
    \begin{align}\label{eq:mck_perturb}
    {\opnorm{\mcK - \mcT}\leq 2\beta(K)\sqrt{\frac2n\log\round{\tfrac4\delta}}}.
\end{align}    
    Similarly, with probability at least $1-\frac\delta2$, with respect to the randomness of $X_s$, we have
\begin{align}\label{eq:mck'_perturb}    
\opnorm{\mcK' - \mcT}\leq 2\beta(K)\sqrt{\frac2s\log\round{\tfrac4\delta}}.
\end{align}
\end{corollary}
The result in \cite[Theorem 7]{rosasco2010learning}  was stated in terms of $\hsnorm{\cdot}$. The above corollary follows by using the fact that $\opnorm{\cdot}\leq\hsnorm{\cdot}.$ Additionally, we substitute their $\tau$ for $\log\round{\tfrac4\delta}$ for the sake of readability.

\begin{proposition}\label{propos:operator_disturb}
    With probability at least $1-\delta$ (with respect to the randomness in $X_s$ and $X_n$),
    \begin{align*}
    \opnorm{\sqrtop{\mcK} - \sqrtop{\mcK'}}
    \leq 2\sqrt{\beta(K)\sqrt{\frac2s\log\round{\tfrac{4}{\delta}}}}.
    \end{align*}
\end{proposition}
\begin{proof}
Observe that we have the bound,
\begin{equation}\label{eq:mckmck'}
    \opnorm{\sqrtop{\mcK}-\sqrtop{\mcK'}}\leq \sqrt{\opnorm{\mc K-\mc K'}}
\end{equation}
following \cite[Theorem X.1, page 290]{bhatia2013matrix}. 
Next, by triangle inequality, we get that with probability $1-\delta$ (since we use a union bound on neither inequality failing, each with failure probability $\frac\delta2$),
\begin{align*}
    &\opnorm{\mcK - \mcK'} \leq \opnorm{\mcT - \mcK'} + \opnorm{\mcK - \mcT} 
    \\&\overset{\rm (a)}{\leq} 2\beta(K)\sqrt{2\log\round{\tfrac4\delta}}\round{\tfrac{1}{\sqrt{s}}+\tfrac{1}{\sqrt{n}}}
    \leq
    4\beta(K)\sqrt{\tfrac2s\log\round{\tfrac4\delta}}.
\end{align*}
where (a) is because of \Cref{cor:rosasco} and the last inequality holds since $s\leq n.$ 
\end{proof}


Recall definition of $\kappa$ in equation \eqref{def:kappa}.
\begin{proposition}\label{prop:ratio_to_condition_number}
Suppose there exists $\gamma>0$ such that for all $f\in\Hilbert_K$
we have
\begin{align}\label{eq:condition_multiplicative_bound}
  (1+\gamma)^{-1}\snorm{\sqrtop{\mc P}_q f}\leq \snorm{\sqrtop{\mc P}_{s,q}f}\leq (1+\gamma){\snorm{\sqrtop{\mc P}_q f}}.
\end{align}
Then
$
    \kappa(\sqrtop{\mc P_{s,q}} \mycirc\mcK\mycirc\sqrtop{\mc P_{s,q}}) \leq (1+\gamma)^4 \kappa({\mcP}_q \mycirc \mcK)
$.
\end{proposition}
\begin{proof}
We start by establishing 
\begin{equation}\label{eq:step1}
\kappa(\sqrtop{\mc P_{s,q}} \mycirc\mcK\mycirc\sqrtop{\mc P_{s,q}}) = \kappa(\sqrtop{{\mcP}_{s,q}} \mycirc \sqrtop{\mcK})^2.
\end{equation}
This follows from the following observation: 
For a finite rank operator $\mcA \in L(\Hilbert_K)$ we have 
\begin{equation}\label{eq:adjoint}
\kappa(\mcA \mycirc \mcA^*) = \kappa(\mcA)^2
\end{equation}
(which can be seen by expressing $\kappa(\cdot)$ in terms of singular values as $\kappa(\mcA) = \sigma_1(\mcA)/\sigma_{\operatorname{rank}(\mcA)}(\mcA)$ and standard matrix arguments, see discussion after \cref{def:kappa}).
Then set $\mcA = \sqrtop{\mc P_{s,q}} \mycirc \sqrtop{\mcK}$ in \eqref{eq:adjoint} to get \eqref{eq:step1}.

A direct application of \eqref{eq:condition_multiplicative_bound} and the observation that $\nullspace(\sqrtop{{\mcP}_{s,q}} \mycirc \sqrtop{\mcK}) = \nullspace(\sqrtop{{\mcP}_{q}} \mycirc \sqrtop{\mcK})$ (using the fact that $\sqrtop{{\mcP}_{s,q}}$ and  $\sqrtop{{\mcP}_{q}}$ are invertible) in the definition of $\kappa(\cdot)$ (\cref{def:kappa}) gives 
\begin{equation}\label{eq:step2}
\kappa(\sqrtop{{\mcP}_{s,q}} \mycirc \sqrtop{\mcK}) \leq (1+\gamma)^2\kappa(\sqrtop{{\mcP}_{q}} \mycirc \sqrtop{\mcK}).
\end{equation}

To conclude, notice that $\sqrtop{{\mcP}_{q}}$ and $\sqrtop{\mcK}$ commute. 
Thus, $\kappa({\mcP}_{q} \mycirc \mcK) = \kappa(\sqrtop{{\mcP}_{q}} \mycirc \sqrtop{{\mcP}_{q}} \mycirc \sqrtop{\mcK}\mycirc \sqrtop{\mcK}) = \kappa(\sqrtop{{\mcP}_{q}} \mycirc \sqrtop{\mcK}\mycirc \sqrtop{\mcK} \mycirc \sqrtop{{\mcP}_{q}}) =   \kappa(\sqrtop{{\mcP}_{q}} \mycirc \sqrtop{\mcK})^2  $ (using \eqref{eq:adjoint} again). The claim follows from this and \cref{eq:step1,eq:step2}.
\end{proof}

We are now ready to complete the proof of \Cref{theorem:main}.

\begin{proof}(of \Cref{theorem:main}).
In \Cref{sec:operator} we provide a general framework that bounds $\hfrac{\snorm{\sqrtop{\mc P}_qf}}{\snorm{\sqrtop{\mc P}_{s,q}f}}$ close to $1$, via a bound on $\opnorm{\mc K-\mc K'}$. 
In the special case, in \cref{proposition:operator_norm_epsilon} if we set $\mc V$ as $\sqrtop{\mc K}$, and $\mc V'$ as $\sqrtop{\mc K'}$, and the observation that $\sqrtop{\mc P}_q=\sqrt{\lambda_q}h_{\sqrt{\lambda_q}}(\sqrtop{\mc K})^{-1}$,  and $\sqrtop{\mc P}_{s,q}=\sqrt{\lambda_q'}h_{\sqrt{\lambda_q'}}(\sqrtop{\mc K'})^{-1}$, we get the following,
\begin{align}\label{eq:def:zeta_bound}
    (1+\zeta)^{-1} \snorm{\sqrtop{\mcP}_q f} \leq \snorm{\sqrtop{\mcP}_{s,q} f} \leq \round{1+\zeta}\snorm{\sqrtop{\mcP}_qf},\\
\label{eq:def:zeta}
\zeta := 2C\log (n+1)\cdot \tfrac{\opsnorm{\sqrtop{\mcK}-\sqrtop{\mcK'}}}{\sqrt{\lambda_q}}\left(1+\tfrac{\opsnorm{\sqrtop{\mc K}}}{\sqrt{\lambda_q'}}\right)
\end{align}
for some universal constant $C$. Note that we used the fact that $C\log (\rk(\mcK) + \rk(\mcK')+1)\leq C\log (2n+1) \leq 2C\log (n+1)$.

Observe that \Cref{prop:ratio_to_condition_number} says if we can find a suitable upper bound for $\zeta$, we are done. Indeed, we will show that choosing $s$ as in equation \eqref{eq:s_bound} leads to $\zeta\leq \eps$.



First, we assert that for sufficiently large values of \(n\) and \(s\), we can eliminate the dependence on samples and relate everything to the operator \(\mathcal{T}\). 
Note that by Weyl's inequality we have $\abs{\lambda_q -\lambda_q^*} \leq \opnorm{\mc K-\mc T}$ (and similarly for $\lambda_q'$).
Hence 
via equation
\eqref{eq:mck_perturb} we have,
\begin{align*}
    \lambda_q&\geq\lambda_q^*-\opnorm{\mc K-\mc T}\geq\lambda_q^*-2\beta(K)\sqrt{\tfrac2n\log\round{\tfrac4\delta}}.
\end{align*}
and similarly via equation \eqref{eq:mck'_perturb} we have,
$
\lambda_q'\geq\lambda_q^*-2\beta(K)\sqrt{\tfrac2s\log\round{\tfrac4\delta}}.
$
Using $n\geq s\geq c_1C_{K,q}^2\log\round{\tfrac4\delta}$, where $c_1:=32$ and $C_{K,q}=\beta(K)/\lambda_q^*$, we can write,  
\begin{equation}\label{eq:populationlambda}
    \frac{\lambda_q^*}{2}\leq \lambda_q'\leq\frac{3\lambda_q^*}{2},\quad\text{and}\quad \frac{\lambda_q^*}{2}\leq\lambda_q\leq\frac{3\lambda_q^*}{2}.
\end{equation}
which hold together with probability at least $1-\delta,$ (via a union bound on the two events in Corollary \ref{cor:rosasco}). 

We find an upper bound for $\zeta$ from \eqref{eq:def:zeta} as following,
\begin{align*}
    &\tfrac{\zeta}{2C\log(1+n)} 
    \overset{\rm (a)}{\leq} \sqrt{\frac{4\beta(K)}{\lambda_q}\sqrt{\frac2s\log\round{\tfrac{4}{\delta}}}}\left(1+\frac{\opsnorm{\sqrtop{\mc K}}}{\sqrt{\lambda_q'}}\right)\\
    &\overset{\rm (b)}{\leq} \sqrt{\frac{4\beta(K)}{\lambda_q}\sqrt{\frac2s\log\round{\tfrac{4}{\delta}}}}\left(1+\sqrt{\frac{\beta(K)}{\lambda_q'}}\right)\\
    &\overset{\rm (c)}{\leq} \sqrt{\frac{4\beta(K)}{\lambda_q}\sqrt{\frac2s\log\round{\tfrac{4}{\delta}}}}\left(1+\sqrt{2}\sqrt{\frac{\beta(K)}{\lambda_q^*}}\right)\\
    &\overset{\rm (d)}{\leq} \sqrt{\frac{4\beta(K)}{\lambda_q}\sqrt{\frac2s\log\round{\tfrac{4}{\delta}}}}\left(2\sqrt{\frac{\beta(K)}{\lambda_q^*}}\right)\\    
    &\overset{\rm (e)}\leq
    \frac{4\sqrt2\beta(K)}{\lambda_q^*}\left(\frac{{2\log\round{\tfrac{4}{\delta}}}}{{s}}\right)^{\frac14}
    \overset{\rm (f)}{\leq}\frac{\eps}{2C\log(1+n)}
\end{align*}
where (a) is due to \Cref{propos:operator_disturb}, (c) and (e) apply inequalities \eqref{eq:populationlambda}, and (b) is justified by the following,
\begin{align*}
    \opsnorm{\mcK^{\frac12}} &= \opnorm{\mc K}^{\frac12}\leq \hsnorm{\mcK}^{\frac12}
\leq\sqrtop{\trace(\mcK)}\leq\sqrt{\beta(K)},
\end{align*}
where in we have used the monotonicity of the Schatten norm, and inequality \eqref{eq:trace_bound}. Similarly, inequality (d) holds because $\lambda_q^*\leq \trace(\mc T)\leq \beta(K)$, again following \eqref{eq:trace_bound}.
Finally, (f) holds because for $c_2:=2^{16}C^4$, we have assumed 
$s \geq c_2 C_{K,q}^4
\frac{ \log^4 n}{\eps^4}\log\left(\tfrac{4}{\delta}\right),$ whereby we can conclude $\zeta\leq \eps.$

In conclusion, we have established that with probability at least $1-\delta$, under the assumptions in the statement of Theorem \ref{theorem:main}, equation \eqref{eq:def:zeta_bound} holds. \Cref{prop:ratio_to_condition_number} (with $\gamma\gets\zeta$), along with $\zeta\leq\eps$, proves the claim.
\end{proof}

\mysection{From Additive to Multiplicative Approximation}\label{sec:operator}
In this section, we show that when two finite rank operators $\mcV$ and $\mcV'$ are close in the operator norm, their corresponding preconditioners are similarly close \emph{in a multiplicative sense}.
The results in this section do not use properties of RKHS and hold for any separable Hilbert space $\Hilbert$.

We will prove the following general lemma.

\begin{lemma}\label{proposition:operator_norm_epsilon}
Suppose $\mc V, \mcV' \in L(\Hilbert)$ are finite rank, self-adjoint positive semidefinite operators. Eigenvalues of $\mcV$ are ordered as $\nu_1 \geq \nu_2 \geq \dotsb$ and those of $\mcV'$ are $\nu_1' \geq \nu_2' \geq \dotsb$. 
Further assume that $q \in \NN$ is such that $\nu_{q}, \nu_{q}' > 0$.
Define $\mcC, \mcC' \in L(\Hilbert)$ given by $\mcC = \nu_{q} h_{\nu_{q}}(\mc V)^{-1}$ and $\mcC' = \nu_{q}' h_{\nu_{q}'}(\mcV')^{-1}$.
Then
\[
(\forall f\neq 0) \quad \frac{\norm{\mcC f}}{\norm{\mcC' f}} \in [(1+\eps')^{-1},1+\eps']
\]
where $\eps' 
= C \log \left(\rk(\mcV) + \rk(\mcV')+1\right) \frac{\opnorm{\mcV - \mcV'}}{\nu_{q}}(1 + \frac{\opnorm{\mcV'}}{\nu_{q}'})
$ for some universal constant $C$.
\end{lemma}


The first step, \cref{lemma:rationbound}, states that it suffices to find an upper bound for $\opnorm{\mcV-\mcV'}$ and a lower bound for the eigenvalues of $\mcC$ and $\mcC'$. 

\begin{lemma}\label{lemma:rationbound}
Let $\mathcal{A}, \mathcal{B} \in L(\Hilbert)$ be such that for all $f \in \Hilbert$, $\norm{\mcA f} \geq \lambdamin \norm{f}$ and $\norm{\mcB f} \geq \lambdamin \norm{f}$. 
Additionally, assume $\opnorm{\mathcal{A}-\mathcal{B}} \leq \eps$. 
Then for all $f \in \Hilbert$ we have
$(1+\frac{\eps}{\lambdamin})^{-1}\norm{\mcB f } \leq \norm{\mcA f} \leq (1+\frac{\eps}{\lambdamin})\norm{\mcB f }$.
\end{lemma}
\begin{proof}
By symmetry it is enough to prove the right hand side inequality:  $\norm{\mcA f} = \norm{\mcA f - \mcB f + \mcB f} \leq \norm{\mcB f} + \eps \norm{f} \leq \norm{\mcB f} (1+\frac{\eps}{\lambdamin})$.
\end{proof}

The next step, \cref{lem:equivalences}, is the observation that the quality of the multiplicative approximation in \cref{lemma:rationbound} is invariant under taking inverses of the operators.
It turns out that it will be easier to bound $\opnorm{\mcC^{-1} - \mcC'^{-1}}$ and apply \cref{lemma:rationbound} to $\mcC^{-1}$ and $\mcC'^{-1}$.
As a bonus, in that case we can take $\lambdamin = 1$.


\begin{lemma}\label{lem:equivalences}
Let $\mcA, \mcB \in L(\Hilbert)$ be self-adjoint and invertible.
Let $c \geq 1$.
The following statements are equivalent:
\begin{enumerate}
\item $(\forall f\neq 0) \frac{\norm{\mcA f}}{\norm{\mcB f}} \in [c^{-1},c]$.
\item $\opnorm{\mcA\mycirc \mcB^{-1}}, \opnorm{\mcB\mycirc \mcA^{-1}}  \in [c^{-1},c]$.
\item $\opnorm{\mcB^{-1}\mycirc \mcA}, \opnorm{\mcA^{-1}\mycirc \mcB}  \in [c^{-1},c]$.
\item $(\forall f\neq 0) \frac{\norm{\mcA^{-1} f}}{\norm{\mcB^{-1} f}} \in [c^{-1},c]$.
\end{enumerate}
\end{lemma}
\begin{proof}
($1 \Leftrightarrow 2$) 
In 1, use substitution $f = \mcB^{-1} (g)$ to get $(\forall g\neq 0) \frac{\norm{\mcA \mycirc\mcB^{-1} (g)}}{\norm{g}} \in [c^{-1},c]$.
Taking $\sup$ we get $\opnorm{\mcA \mycirc\mcB^{-1}} \in [c^{-1},c]$.
Similarly, in 1 use substitution $f = \mcA^{-1} g$ and take $\inf$ to get $\opnorm{\mcB^{-1} \mycirc\mcA} \in [c^{-1},c]$.
The same argument in reverse shows the equivalence.

($2 \Leftrightarrow 3$) 
Follows immediately from the fact that the operator norm is invariant under taking adjoint.

($3 \Leftrightarrow 4$)
This follows from $1 \Leftrightarrow 2$ with $\mcA^{-1}$ in the role of $\mcA$ and $\mcB^{-1}$ in the role of $\mcB$.
\end{proof}

\begin{proposition}\label{cor:bhatia}
For any two operators $\mc S, \mc T \in L(\Hilbert)$ where $\Hilbert$ is an $k$-dimensional Hilbert space we have
\[
\opnorm{h_0(\mc S) - h_0(\mc T)} \leq  C\log (k+1) \opnorm{\mc S -\mc T}.
\]
where $C>0$ is a universal constant.
\end{proposition}
\begin{proof}
For an operator $\mc S$, denote by $|\mc S|$ the operator $\sqrtop{(\mc S^*\mc S)}$.
\cite[Theorem 4.2]{MR2650102} states that $ \opnorm{\abs{\mc S} - \abs{\mc T}} \leq O(\log k) \opnorm{\mc S-\mc T}\leq c\log (k+1) \opnorm{\mc S-\mc T}$ for some universal constant $c>0$.
Thus
we have $h_0(\mc S) = (\mc S+|\mc S|)/2$ whereby,
\begin{align*}
&\opnorm{h_0(\mc S) - h_0(\mc T)} 
= \opnorm{(\mc S+|\mc S|)/2 - (\mc T+|\mc T|)/2} \\
&\leq \frac{1}{2} (\opnorm{\mc S-\mc T} + \opnorm{|\mc S|-|\mc T|})\\
&\leq C \log (k+1)\opnorm{\mc S -\mc  T}.\qedhere
\end{align*}
\end{proof}
Instead of \cite[Theorem 4.2]{MR2650102}, it may be possible to obtain a result similar to \cref{cor:bhatia} by using \cite[result I]{10.3792/pja/1195519395}, which gives a bound where the dependence on $k$ is replaced by a dependence on the operator norms of $\mc S$ and $\mc T$.

\begin{proof}(of \Cref{proposition:operator_norm_epsilon})
%
%
Let $V = \image(\mcV) + \image(\mcV')$, i.e., $V$ is the subspaces generated by taking the linear combination of these two subspaces of $\Hilbert$.
Below, if $\mcA \in L(\Hilbert)$ is such that $\mcA|_V$ has image in $V$, then (by a slight abuse of notation) we consider $\mcA|_V$ as an element of $L(V)$ (even though in general $\mcA|_V:V \to \Hilbert$).
We have 
{
\renewcommand{\ophnorm}[1]{\norm{#1}_{\mathsf{OP}}}
\renewcommand{\opvnorm}[1]{\norm{#1}_{\mathsf{OP}}}
\begin{align*}
&\ophnorm{\mcC^{-1} - \mcC'^{-1}}
= \ophnorm{\rfrac{1}{\nu_{q}} h_{\nu_{q}}(\mcV) - \rfrac{1}{\nu_{q}'} h_{\nu_{q}'}(\mcV')} \\
&= \opvnorm{\rfrac{1}{\nu_{q}} h_{\nu_{q}}(\mcV|_{V}) - \rfrac{1}{\nu_{q}'} h_{\nu_{q}'}(\mcV'|_{V})} \\
&= \opvnorm{\rfrac{1}{\nu_{q}} h_{0}\bigl((\mcV - \nu_{q}\mcI)|_{V}\bigr) - \rfrac{1}{\nu_{q}'} h_{0}\bigl((\mcV' - \nu_{q}'\mcI)|_{V}\bigr)} \\
&= \opvnorm{ h_{0}\bigl((\rfrac{1}{\nu_{q}}\mcV - \mcI)|_{V}\bigr) - h_{0}\bigl((\rfrac{1}{\nu_{q}'} \mcV' - \mcI)|_{V}\bigr)} \\
&\overset{\rm (a)}{\leq}  C\log(\dim(V)+1)\opvnorm{ \rfrac{1}{\nu_{q}}\mcV - \rfrac{1}{\nu_{q}'} \mcV' } \\
%
&\overset{\rm (b)}{\leq} C\log(\dim(V)+1)\biggl( \ophnorm{ \frac{\mcV - \mcV'}{\nu_{q}}}+ \ophnorm{\frac{\nu_{q}' - \nu_{q}}{\nu_{q}\nu_{q}'}\mcV' }\biggr)  \\
&\leq \tfrac{C}{\nu_{q}}\log (\dim(V)+1) \round{\ophnorm{\mathcal{V} - \mathcal{V}'} + \abs{\nu_q'-\nu_q}\frac{\ophnorm{\mcV'}}{\nu_{q}'}}\\
&\leq  \tfrac{C}{\nu_{q}}\log(\dim(V)+1) \ophnorm{\mathcal{V} - \mathcal{V}'} \left(1 + \tfrac1{\nu_{q}'}\ophnorm{\mcV'}\right).
\end{align*}
}
Note that (a) holds because of \Cref{cor:bhatia}, and (b) follows from a triangle inequality. 
From \cref{lemma:rationbound} (with $\lambdamin=1$) we get 
\begin{align*}
(\forall f\neq 0) \quad \frac{\norm{\mcC^{-1} f}}{\norm{\mcC'^{-1} f}} \in [(1+\eps')^{-1},1+\eps'].
\end{align*}
\Cref{lem:equivalences} and the observation that $\dim(V) \leq \rk(\mathcal{V}) + \rk(\mathcal{V'})$ complete the proof.
\end{proof}



\mysection{Conclusion}

In this paper we analyzed the trade-offs of an approximation method for preconditioning used in fast algorithms for training kernel models. Our analysis provides sufficient conditions on the approximation properties of the approximated preconditioner. This analysis guides the design of practical preconditioners using in implementations such as EigenPro2 \cite{ma2019kernel} and EigenPro3 \cite{abedsoltan2023toward} based on preconditioned gradient descent.

\subsection*{Acknowledgements}
A.A. and M.B. are supported by the National Science Foundation (NSF) and the Simons Foundation for the Collaboration on the Theoretical Foundations of Deep Learning (\url{https://deepfoundations.ai/}) through awards DMS-2031883 and \#814639  and the TILOS institute (NSF CCF-2112665). P.P. was supported by a Simons Postdoctoral Fellowship via HDSI at UCSD.
L.R. is supported by the National Science Foundation under Grant CCF-2006994 and acknowledges support by HDSI, UCSD.

\bibliographystyle{apalike}
\bibliography{aux_files/ref}

\end{document}